\def\year{2016}\relax
\documentclass[letterpaper]{article}
\pdfoutput=1
\usepackage{aaai16}
\usepackage{times}
\usepackage{helvet}
\usepackage{courier}
\usepackage{amsmath}
\usepackage{amsfonts}
\usepackage{amsthm}
\usepackage{graphicx}
\usepackage{subfig}
\usepackage{epstopdf}
\usepackage{algorithm}
\usepackage{algorithmic}
\usepackage{url}
\usepackage[toc,page]{appendix}

\newtheorem{definition}{Definition}
\newtheorem{theorem}{Theorem}

\begin{document}
	%
	\title{On the Depth of Deep Neural Networks: A Theoretical View}
	\author{Shizhao Sun$^{1,}$\thanks{This work was done when the author was visiting Microsoft Research Asia.}, Wei Chen$^2$, Liwei Wang$^3$, Xiaoguang Liu$^1$ \and Tie-Yan Liu$^2$ \\
	$^1$College of Computer and Control Engineering, Nankai University, Tianjin, 300071, P. R. China\\
	$^2$Microsoft Research, Beijing, 100080, P. R. China \\
	$^3$Key Laboratory of Machine Perception (MOE), School of EECS, Peking University, Beijing, 100871, P. R. China \\
	sunshizhao@mail.nankai.edu.cn, wche@microsoft.com, wanglw@cis.pku.edu.cn\\
	liuxg@nbjl.nankai.edu.cn, tyliu@microsoft.com
	}
	\maketitle
	\begin{abstract}
		People believe that depth plays an important role in success of deep neural networks (DNN). However, this belief lacks solid theoretical justifications as far as we know. We investigate role of depth from perspective of margin bound. In margin bound, expected error is upper bounded by empirical margin error plus Rademacher Average (RA) based capacity term.  First, we derive an upper bound for RA of DNN, and show that it increases with increasing depth. This indicates negative impact of depth on test performance. Second, we show that deeper networks tend to have larger representation power (measured by Betti numbers based complexity) than shallower networks in multi-class setting, and thus can lead to smaller empirical margin error. This implies positive impact of depth. The combination of these two results shows that for DNN with restricted number of hidden units, increasing depth is not always good since there is a tradeoff between positive and negative impacts. These results inspire us to seek alternative ways to achieve positive impact of depth, e.g., imposing margin-based penalty terms to cross entropy loss so as to reduce empirical margin error without increasing depth. Our experiments show that in this way, we achieve significantly better test performance.
	\end{abstract}
	\section{Introduction}
	\label{introduction}
	Deep neural networks (DNN) have achieved great practical success in many machine learning tasks, such as speech recognition, image classification, and natural language processing~\cite{hinton2006reducing,krizhevsky2012imagenet,hinton2012deep,ciresan2012multi,weston2012deep}. Many people believe that the depth plays an important role in the success of DNN ~\cite{srivastava2015training,simonyan2014very,lee2014deeply,romero2014fitnets,he2015delving,szegedy2014going}. However, as far as we know, such belief is still lacking solid theoretical justification.
	
	On one hand, some researchers have tried to understand the role of depth in DNN by investigating its generalization bound. For example, in~\cite{bartlett1998almost,karpinski1995polynomial,goldberg1995bounding}, generalization bounds for multi-layer neural networks were derived based on Vapnik-Chervonenkis (VC) dimension. In~\cite{bartlett1998sample,koltchinskii2002empirical}, a margin bound was given to fully connected neural networks in the setting of binary classification. In~\cite{neyshabur2015norm}, the capacity of different norm-constrained feed-forward networks was investigated. While these works shed some lights on the theoretical properties of DNN, they have limitations in helping us understand the role of depth, due to the following reasons. First, the number of parameters in many practical DNN models could be very large, sometimes even larger than the size of training data. This makes the VC dimension based generalization bound too loose to use. Second, practical DNN are usually used to perform multi-class classifications and often contains many convolutional layers, such as the model used in the tasks of ImageNet~\cite{deng2009imagenet}. However, most existing bounds are only regarding binary classification and fully connected networks. Therefore, the bounds cannot be used to explain the advantage of using deep neural networks.
	
	On the other hand, in recent years, researchers have tried to explain the role of depth from other angles, e.g., deeper neural networks are able to represent more complex functions. In~\cite{hastad1986almost,delalleau2011shallow}, authors showed that there exist families of functions that can be represented much more efficiently with a deep logic circuit or sum-product network than with a shallow one, i.e., with substantially fewer hidden units. In~\cite{bianchini2014complexity,montufar2014number}, it was demonstrated that deeper nets could represent more complex functions than shallower nets in terms of maximal number of linear regions and Betti numbers. However, these works are apart from the generalization of the learning process, and thus they cannot be used to explain the test performance improvement for DNN. 
	
	To reveal the role of depth in DNN, in this paper, we propose to investigate the margin bound of DNN. According to the margin bound, the expected $0$-$1$ error of a DNN model is upper bounded by the empirical margin error plus a Rademacher Average (RA) based capacity term. Then we first derive an upper bound for the RA-based capacity term, for both fully-connected and convolutional neural networks in the multi-class setting. We find that with the increasing depth, this upper bound of RA will increase, which indicates that depth has its negative impact on the test performance of DNN. Second, for the empirical margin error, we study the representation power of deeper networks, because if a deeper net can produce more complex classifiers, it will be able to fit the training data better w.r.t. any margin coefficient. Specifically, we measure the representation power of a DNN model using the Betti numbers based complexity~\cite{bianchini2014complexity}, and show that, in the multi-class setting, the Betti numbers based complexity of deeper nets are indeed much larger than that of shallower nets. This, on the other hand, implies the positive impact of depth on the test performance of DNN. By combining these two results, we can come to the conclusion that for DNN with restricted number of hidden units, arbitrarily increasing the depth is not always good since there is a clear tradeoff between its positive and negative impacts. In other words, with the increasing depth, the test error of DNN may first decrease, and then increase. This pattern of test error has been validated by our empirical observations on different datasets. 
	
	The above theoretical findings also inspire us to look for alternative ways to achieve the positive impact of depth, and avoid its negative impact. For example, it seems feasible to add a margin-based penalty term to the cross entropy loss of DNN so as to directly reduce the empirical margin error on the training data, without increasing the RA of the DNN model. For ease of reference, we call the algorithm minimizing the penalized cross entropy loss \emph{large margin DNN} (LMDNN)\footnote{One related work is \cite{li2015max}, which combines the generative deep learning methods (e.g., RBM) with a margin-max posterior. In contrast, our approach aims to enlarge the margin of discriminative deep learning methods like DNN.}. We have conducted extensive experiments on benchmark datasets to test the performance of LMDNN. The results show that LMDNN can achieve significantly better test performance than standard DNN. In addition, the models trained by LMDNN have smaller empirical margin error at almost all the margin coefficients, and thus their performance gains can be well explained by our derived theory.
	
	The remaining part of this paper is organized as follows. In Section~\ref{sec:def}, we give some preliminaries for DNN. In Section~\ref{sec:generalization}, we investigate the roles of depth in RA and empirical margin error respectively. In Section~\ref{sec:alg}, we propose the large margin DNN algorithms and conduct experiments to test their performances. In Section~\ref{sec:conclusion}, we conclude the paper and discuss some future works.

	\section{Preliminaries}\label{sec:def}
	Given a multi-class classification problem, we denote $\mathcal{X}=\mathbb{R}^d$ as the input space, $\mathcal{Y}=\{1,\cdots,K\}$ as the output space, and $P$ as the joint distribution over $\mathcal{X}\times\mathcal{Y}$. Here $d$ denotes the dimension of the input space, and $K$ denotes the number of categories in the output space. We have a training set $S=\{(x_1,y_1),\cdots,(x_m,y_m)\}$, which is i.i.d. sampled from $\mathcal{X}\times\mathcal{Y}$ according to distribution $P$. The goal is to learn a prediction model $f\in\mathcal{F}:\mathcal{X}\times\mathcal{Y}\to \mathbb{R}$ from the training set, which produces an output vector $(f(x,k);k\in\mathcal{Y})$ for each instance $x\in\mathcal{X}$ indicating its likelihood of belonging to category $k$. Then the final classification is determined by $\arg\max_{k\in\mathcal{Y}} f(x,k)$. This naturally leads to the following definition of the \emph{margin} $\rho(f;x,y)$ of the model $f$ at a labeled sample $(x,y)$:
	\begin{small}
		\begin{equation}
		\rho(f;x,y)=f(x,y)-\max_{k\neq y} f(x,k).
		\end{equation}
	\end{small}
	The classification accuracy of the prediction model $f$ is measured by its expected $0$-$1$ error, i.e.,
	\begin{small}
		\begin{align}
		err_P(f)&=\Pr_{(x,y)\sim P}\mathbb{I}_{[\arg\max_{k\in \mathcal{Y}}f(x,k)\neq y]}\\
		&=\Pr_{(x,y)\sim P}\mathbb{I}_{[\rho(f;x,y)< 0]},
		\end{align}
	\end{small}where $\mathbb{I}_{[\cdot]}$ is the indicator function.
	
	We call the $0$-$1$ error on the training set \emph{training error} and that on the test set \emph{test error}. Since the expected $0$-$1$ error cannot be obtained due to the unknown distribution $P$, one usually uses the test error as its proxy when examining the classification accuracy.
	
	Now, we consider using neural networks to fulfill the multi-class classification task. Suppose there are $L$ layers in a neural network, including $L-1$ hidden layers and an output layer. There are $n_l$ units in layer $l$ ($l=1,\dots,L$). The number of units in the output layer is fixed by the classification problem, i.e., $n_L=K$. There are weights associated with the edges between units in adjacent layers of the neural network. To avoid over fitting, people usually constraint the size of the weights, e.g., impose a constraint $A$ on the sum of the weights for each unit. We give a unified formulation for both fully connected and convolutional neural networks. Mathematically, we denote the function space of multi-layer neural networks with depth $L$, and weight constraint $A$ as $\mathcal{F}_A^L$, i.e.,
	\begin{small}
		\begin{align}\nonumber
		\mathcal{F}_A^L=\Big\{(x,k)\to\sum_{i=1}^{n_{L-1}}w_if_i(x); f_i\in\mathcal{F}_A^{L-1},\\ 
		\sum_{i=1}^{n_{L-1}}|w_i|\le A,w_i\in \mathbb{R}\Big\};
		\end{align}
	\end{small} 
	for $l=1,\cdots,L-1$,
	\begin{small}
		\begin{align}\nonumber
		\mathcal{F}_A^l=\Big\{x\to\varphi\Big(\phi(f_1(x)),\cdots,\phi(f_{p_l}(x))\Big);\\\label{eqn:pooling}
		f_1,\cdots,f_{p_l}\in\bar{\mathcal{F}}_A^l\Big\},
		\end{align}
		\begin{equation}\label{eqn:weight_sum}
		\bar{\mathcal{F}}_A^l=\Big\{x \to \sum_{i=1}^{n_{l-1}}w_if_i(x);f_{i}\in\mathcal{F}_A^{l-1},\sum_{i=1}^{n_{l-1}}|w_{i}|\le A, w_i \in \mathbb{R}\Big\};
		\end{equation}
	\end{small} 
	and,
	\begin{small}
		\begin{equation}
		\mathcal{F}_A^0=\Big\{x\to x_{|i}; i\in\{1,\cdots,d\} \Big\};
		\end{equation}
	\end{small} 
	where $w_i$ denotes the weight in the neural network, $x_{|i}$ is the $i$-th dimension of input $x$. The functions $\varphi$ and $\phi$ are defined as follows:
	
	(1) If the $l$-th layer is a convolutional layer, the outputs of the $(l-1)$-th layer are mapped to the $l$-th layer by means of filter, activation, and then pooling. That is, in Eqn (\ref{eqn:weight_sum}), lots of weights equal $0$, and $n_l$ is determined by $n_{l-1}$ as well as the number and domain size of the filters. In Eqn (\ref{eqn:pooling}), $p_{l}$ equals the size of the pooling region in the $l$-th layer, and function $\varphi:\mathbb{R}^{p_l}\to\mathbb{R}$ is called the \emph{pooling function}. Widely-used pooling functions include the max-pooling $\max(t_1,\cdots,t_{p_l})$ and the average-pooling $(t_1+\cdots+t_{p_l})/p_l$. Function $\phi$ is increasing and usually called the \emph{activation function}. Widely-used activation functions include the standard sigmoid function $\phi(t)=\frac{1}{1+e^{-t}}$, the tanh function $\phi(t)=\frac{e^t-e^{-t}}{e^t+e^{-t}}$, and the rectifier function $\phi(t)=\max(0,t)$. Please note that all these activation functions are $1$-Lipschitz.
	
	(2) If the $l$-th layer is a fully connected layer, the outputs of the $(l-1)$-th layer are mapped to the $l$-th layer by linear combination and subsequently activation. That is, in Eqn (\ref{eqn:pooling}) $p_{l}=1$ and $\varphi(x)=x$.
	
	Because distribution $P$ is unknown and the $0$-$1$ error is non-continuous, a common way of learning the weights in the neural network is to minimize the empirical (surrogate) loss function. A widely used loss function is the cross entropy loss, which is defined as follows,
	\begin{small}
		\begin{equation}
		C(f;x,y)=-\sum_{k=1}^K z_{k} \ln \sigma(x,k),
		\end{equation}
	\end{small}where $z_{k}=1$ if $k=y$, and $z_{k}=0$ otherwise. Here $\sigma(x,k)=\frac{\exp(f(x,k))}{\sum_{j=1}^{K}\exp(f(x,j))}$ is the softmax operation that normalizes the outputs of the neural network to a distribution.
	
	Back-propagation algorithm is usually employed to minimize the loss functions, in which the weights are updated by means of stochastic gradient descent (SGD).

	\section{The Role of Depth in Deep Neural Networks}
	\label{sec:generalization}
	In this section, we analyze the role of depth in DNN, from the perspective of the margin bound. For this purpose, we first give the definitions of empirical margin error and Rademacher Average (RA), and then introduce the margin bound for multi-class classification.
	\begin{definition}
		Suppose $f\in\mathcal{F}:\mathcal{X}\times\mathcal{Y}\to\mathbb{R}$ is a multi-class prediction model. For $\forall \gamma>0$, the empirical margin error of $f$ at margin coefficient $\gamma$ is defined as follows:
		\begin{small}
			\begin{equation}
			err_S^\gamma(f)=\frac{1}{m}\sum_{i=1}^m \mathbb{I}_{[\rho(f;x_i,y_i) \le\gamma]}.
			\end{equation}
		\end{small} 
	\end{definition} 
	\begin{definition}
		Suppose $\mathcal{F}:\mathcal{X}\to\mathbb{R}$ is a model space with a single dimensional output. The Rademacher average (RA) of $\mathcal{F}$ is defined as follows:
		\begin{small}
			\begin{equation}
			R_m(\mathcal{F})=\mathbf{E}_{\mathbf{x},\mathbf{\sigma}}\Big[\sup_{f\in\mathcal{F}}\Big|\frac{2}{m}\sum_{i=1}^{m}\sigma_i f(x_i)\Big|\Big],
			\end{equation}
		\end{small} 
		where $\mathbf{x}=\{x_1,\cdots,x_m\}\sim P_x^m$, and $\{\sigma_1,\cdots,\sigma_m\}$ are i.i.d. sampled with $P(\sigma_i=1)=1/2, P(\sigma_i=-1)=1/2$.
	\end{definition}
	\begin{theorem}\cite{koltchinskii2002empirical}\label{thm:generalization}
		Suppose $f\in\mathcal{F}:\mathcal{X}\times\mathcal{Y}\to\mathbb{R}$ is a multi-class prediction model. For $\forall \delta>0$, with probability at least $1-\delta$, we have, $\forall f\in\mathcal{F}$,
		\begin{small}
			\begin{align}
			\label{eqn:generalization_nn}\nonumber
			err_P(f)\le & \inf_{\gamma>0}\Big\{err_S^\gamma(f) + \frac{8K(2K-1)}{\gamma} R_m(\tilde{\mathcal{F}}) \\
			& +\sqrt{\frac{\log\log_2(2\gamma^{-1})}{m}} +\sqrt{\frac{\log(2\delta^{-1})}{2m}}\Big\}.
			\end{align}
		\end{small} 
		where $\tilde{\mathcal{F}}=\{x\to f(\cdot,k);k\in\mathcal{Y},f\in\mathcal{F}\}$
	\end{theorem} 
	According to the margin bound given in Theorem \ref{thm:generalization}, the expected $0$-$1$ error of a DNN model can be upper bounded by the sum of two terms, RA and the empirical margin error. In the next two subsections, we will make discussions on the role of depth in these two terms, respectively. 
	
	\subsection{Rademacher Average}\label{subsec:margin_theory}
	In this subsection, we study the role of depth in the RA-based capacity term. 
	
	In the following theorem, we derive an uniform upper bound of RA for both the fully-connected and convolutional neural networks.\footnote{To the best of our knowledge, an upper bound of RA for fully connected neural networks has been derived before~\cite{bartlett2003rademacher,neyshabur2015norm}, but there is no result available for the convolutional neural networks.}
	
	\begin{theorem}\label{thm:generalization_nn}
		Suppose input space $\mathcal{X}=[-M,M]^d$. In the deep neural networks, if activation function $\phi$ is $L_\phi$- Lipschitz and non-negative, pooling function $\varphi$ is max-pooling or average-pooling, and the size of pooling region in each layer is bounded, i.e., $p_l\leq p$, then we have,
		\begin{small}
			\begin{equation} \label{eqn:rademacher_bound}
			R_m(\mathcal{F}_A^L) \le cM \sqrt{\frac{\ln d}{m}}(pL_\phi A)^{L}.
			\end{equation}
		\end{small} 
		where $c$ is a constant.
	\end{theorem} 
	\begin{proof} 
		According to the definition of $\mathcal{F}_A^L$ and RA, we have,
		\begin{small}
			\begin{align*}
			R_m(\mathcal{F}_A^L)
			= \mathbf{E}_{\mathbf{x},\mathbf{\sigma}}\Big[\sup_{\|\mathbf{w}\|_1\le A,f_j\in\mathcal{F}_A^{L-1}}\Big|\frac{2}{m}\sum_{i=1}^{m}\sigma_i\sum_{j=1}^{n_{L-1}}w_jf_j(x_i)\Big|\Big] \\
			= \mathbf{E}_{\mathbf{x},\mathbf{\sigma}}\Big[\sup_{\|\mathbf{w}\|_1\le A,f_j\in\mathcal{F}_A^{L-1}}\Big|\frac{2}{m}\sum_{j=1}^{n_{L-1}}w_j\sum_{i=1}^{m}\sigma_if_j(x_i)\Big|\Big].
			\end{align*}
		\end{small} 
		Supposing {$\mathbf{w}=\{w_1,\cdots,w_{n_{L-1}}\}$} and {$\mathbf{h}=\{\sum_{i=1}^{m}\sigma_i f_1(x_i),\cdots,\sum_{i=1}^{m}\sigma_if_{n_{L-1}}(x_i)\}$}, the inner product {$\langle\mathbf{w},\mathbf{h}\rangle$} is maximized when $\mathbf{w}$ is at one of the extreme points of the $l_1$ ball, which implies:
		\begin{small}
			\begin{align}\nonumber
			R_m(\mathcal{F}_A^L) 
			&\le A \mathbf{E}_{\mathbf{x},\mathbf{\sigma}}\Big[\sup_{f\in\mathcal{F}_A^{L-1}}\Big|\frac{2}{m}\sum_{i=1}^{m}\sigma_i f(x_i)\Big|\Big] \\
			&=A R_m(\mathcal{F}_A^{L-1}).\label{eqn:max}
			\end{align}
		\end{small} 
		For function class $\mathcal{F}_A^{L-1}$, if the ($L-1$)-th layer is a fully connected layer, it is clear that $R_m(\mathcal{F}_A^{L-1})\le R_m(\phi\circ \bar{\mathcal{F}}_A^{L-1})$ holds. If the ($L-1$)-th layer is a convolutional layer with max-pooling or average-pooling, we have,
		\begin{small}
			\begin{align}\nonumber
			& R_m(\mathcal{F}_A^{L-1})\\ \nonumber
			&\le\mathbf{E}_{\mathbf{x},\mathbf{\sigma}}\Big[\sup_{f_1,\cdots,f_{p_{L-1}}\in\bar{\mathcal{F}}_A^{L-1}}\Big|\frac{2}{m}\sum_{i=1}^m\sigma_i \sum_{j=1}^{p_{L-1}}\phi(f_j(x_i))\Big|\Big] \label{eqn:pooling_less}  \\ 
			& =  p_{L-1}R_m(\phi\circ\bar{\mathcal{F}}_A^{L-1}).
			\end{align}
		\end{small} 
		The inequality (\ref{eqn:pooling_less}) holds due to the fact that most widely used activation functions $\phi$ (e.g., standard sigmoid and rectifier) have non-negative outputs.
		
		Therefore, for both fully connected layers and convolutional layers, $R_m(\mathcal{F}_A^{L-1})\le p_{L-1} R_m(\phi\circ \bar{\mathcal{F}}_A^{L-1})$ uniformly holds. Further considering the Lipschitz property of $\phi$, we have,
		\begin{small}
			\begin{equation}\label{eqn:lip}
			R_m(\mathcal{F}_A^{L-1}) \le 2p_{L-1}L_\phi R_m(\bar{\mathcal{F}}_A^{L-1}).
			\end{equation}
		\end{small} 
		Iteratively using maximization principle of inner product in (\ref{eqn:max}), property of RA in (\ref{eqn:pooling_less}) and Lipschitz property in (\ref{eqn:lip}), considering $p_l\leq p$, we can obtain the following inequality,
		\begin{small}
			\begin{equation}\label{eqn:repeat}
			R_m(\mathcal{F}_A^{L}) \le (2pL_\phi A)^{L-1}R_m(\bar{\mathcal{F}}_A^1).
			\end{equation}
		\end{small} 
		According to~\cite{bartlett2003rademacher}, $R_m(\bar{\mathcal{F}}_A^1)$ can be bounded by:
		\begin{small}
			\begin{equation}\label{eqn:g_bound}
			R_m(\bar{\mathcal{F}}_A^1) \le cAM\sqrt{\frac{\ln d}{m}},
			\end{equation}
		\end{small}where $c$ is a constant.
		
		Combining (\ref{eqn:repeat}) and (\ref{eqn:g_bound}), we can obtain the upper bound on the RA of DNN.
	\end{proof} 
	From the above theorem, we can see that with the increasing depth, the upper bound of RA will increase, and thus the margin bound will become looser. This indicates that depth has its negative impact on the test performance of DNN.
	
	\subsection{Empirical Margin Error}
	\label{subsec:margin_exp}
	In this subsection, we study the role of depth in empirical margin error. 
	
	To this end, we first discuss representation power of DNN models. In particular, we use the Betti numbers based complexity~\cite{bianchini2014complexity} to measure the representation power. We generalize the definition of Betti numbers based complexity into multi-class setting as follows.
	\begin{definition}
		The Betti numbers based complexity of functions implemented by multi-class neural networks $\mathcal{F}_A^L$ is defined as $N(\mathcal{F}_A^L)=\sum_{i=1}^{K-1}B(S_i)$, where $B(S_i)$ is the sum of Betti numbers\footnote{For any subset $S\subset \mathbb{R}^d$, there exist $d$ Betti numbers, denoted as $b_j(S), 0\le j \le d-1$. Therefore, the sum of Betti numbers is denoted as $B(S)=\sum_{j=0}^{d-1} b_j(S)$. Intuitively, the first Betti number $b_0(S)$ is the number of connected components of the set $S$, while the $j$-th Betti number $b_j(S)$ counts the number of $(j+1)$-dimension holes in $S$~\cite{bianchini2014complexity}.} that measures the complexity of the set $S_i$. Here $S_i=\cap_{j=1,j\neq i}^{K}\{x\in\mathbb{R}^d\mid f(x,i)-f(x,j)\ge 0;f(x,\cdot)\in\mathcal{F}_A^L\}, i=1,\dots,K-1$.
	\end{definition}
	As can be seen from the above definition, the Betti numbers based complexity considers classification output and merge those regions corresponding to the same classification output (thus is more accurate than the linear region number complexity~\cite{montufar2014number} in measuring the representation power). As far as we know, only for binary classification and fully connected networks, the bounds of the Betti numbers based complexity was derived~\cite{bianchini2014complexity}, and there is no result for the setting of multi-class classification and convolutional networks. In the following, we give our own theorem to fill in this gap.
	\begin{theorem}\label{thm:multi_class_complexity}
		For neural networks $\mathcal{F}_A^L$ that has $h$ hidden units. If activation function $\phi$ is a Pfaffian function with complexity $(\alpha,\beta,\eta)$, pooling function $\varphi$ is average-pooling and $d\le h\eta$, then
		\begin{small}
			\begin{align}\nonumber
			& N(\mathcal{F}_A^L) \le (K-1)^{d+1}2^{h\eta(h\eta-1)/2} \\
			&\times O\left(\left(d\left(\left(\alpha+\beta-1+\alpha\beta\right)(L-1)+\beta\left(\alpha+1\right)\right)\right)^{d+h\eta}\right)
			\end{align}
		\end{small}
	\end{theorem}
	\begin{proof}
		We first show that the functions $f(x,\cdot)\in\mathcal{F}_A^L$ are Pfaffian functions with complexity $((\alpha+\beta-1+\alpha\beta)(L-1)+\alpha\beta,\beta,h\eta)$, where $\mathcal{F}_A^L$ can contain both fully-connected layers and convolutional layers. Assume the Pfaffian chain which defines activation function $\phi(t)$ is $(\phi_1(t),\dots,\phi_\eta(t))$, and then $s^l$ is constructed by applying all $\phi_i,1\le i\le\eta$ on all the neurons up to layer $l-1$, i.e., $f^l\in\bar{\mathcal{F}}_A^l,l\in\{1,\dots,L-1\}$. As the first step, we need to get the degree of $f^l$ in the chain $s^l$. Since $f^l=\frac{1}{p_{l-1}}\sum_{k=1}^{n_{l-1}}w_k(\phi (f_{k,1}^{l-1})+\cdots+\phi(f_{k,p_{l-1}}^{l-1}))$ and $\phi$ is a Pfaffian function, $f^l$ is a polynomial of degree $\beta$ in the chain $s^l$. Then, it remains to show that the derivative of each function in $s^l$, i.e., $\frac{\partial \phi_j(f^l)}{\partial x_{|i}}=\frac{d\phi_j(f^l)}{d f^l}\frac{\partial f^l}{\partial x_{|i}}$, can be defined as a polynomial in the functions of the chain and the input. For average pooling, by iteratively using chain rule, we can obtain that the highest degree terms of $\frac{\partial f^l}{\partial x_{|i}}$ are in the form of $\prod_{i=1}^{l-1}\frac{d\phi(f^i)}{df^i}$. Following the lemma 2 in~\cite{bianchini2014complexity}, we obtain the complexity of $f(x,\cdot)\in\mathcal{F}_A^L$.
		
		Furthermore, the sum of two Pfaffian functions $f_1$ and $f_2$ defined by the same Pffaffian chain of length $\eta$ with complexity $(\alpha_1,\beta_1,\eta)$ and $(\alpha_2,\beta_2,\eta)$ respectively is a Pfaffian function with complexity $\left(\max(\alpha_1,\alpha_2),\max(\beta_1,\beta_2),\eta\right)$~\cite{gabrielov2004complexity}. Therefore, $f(x,i)-f(x,j),i\neq j$ is a Pfaffian function with complexity $((\alpha+\beta-1+\alpha\beta)(L-1)+\alpha\beta,\beta,h\eta)$.
		
		According to~\cite{zell1999betti}, since $S_i$ is defined by $K-1$ sign conditions (inequalities or equalities) on Pfaffian functions, and all the functions defining $S_i$ have complexity at most $((\alpha+\beta+\alpha\beta)(L-1)+\alpha\beta,\beta,h\eta)$, $B(S_i)$ can be upper bounded by $(K-1)^{d}2^{h\eta(h\eta-1)/2}\times O(\left(d\left(\left(\alpha+\beta-1+\alpha\beta\right)(L-1)+\beta\left(\alpha+1\right)\right)\right)^{d+h\eta})$.
		
		Summing over all $i\in\{1,\dots,K-1\}$, we get the results stated in Theorem~\ref{thm:multi_class_complexity}. 
	\end{proof}
	Theorem~\ref{thm:multi_class_complexity} upper bounds the Betti numbers based complexity for general activation functions. For specific active functions, we can get the following results: when $\phi=\arctan(\cdot)$ and $d\le 2h$, since $\arctan$ is of complexity $(3,1,2)$, we have $N(\mathcal{F}_A^L) \le (K-1)^{d+1}2^{h(2h-1)}O((d(L-1)+d)^{d+2h})$; when $\phi=\tanh(\cdot)$ and $n\le h$, since $\tanh$ is of complexity $(2,1,1)$, we have $N(\mathcal{F}_A^L) \le (K-1)^{d+1}2^{h(h-1)/2}O((d(L-1)+d)^{d+h})$.
	
	Basically, Theorem~\ref{thm:multi_class_complexity} indicates that in the multi-class setting, the Betti numbers based complexity grows with the increasing depth $L$. As a result, deeper nets will have larger representation power than shallower nets, which makes deeper nets fit better to the training data and achieve smaller empirical margin error. This indicates that depth has its positive impact on the test performance of DNN.
	
	Actually, above discussions about impact of depth on representation power are consistent with our empirical findings. We conducted experiments on two datasets, MNIST~\cite{lecun1998gradient} and CIFAR-10~\cite{Krizhevsky09learningmultiple}. To investigate the influence of network depth $L$, we trained fully-connected DNN with different depths and restricted number of hidden units. The experimental results are shown in Figure~\ref{fig:depth_margin} and indicate that no matter on which dataset, deeper networks have smaller empirical margin errors than shallower networks for most of the margin coefficients. 
	\begin{figure}
		\centering
		\subfloat[MNIST]
		{
			\includegraphics[width=0.47\columnwidth]{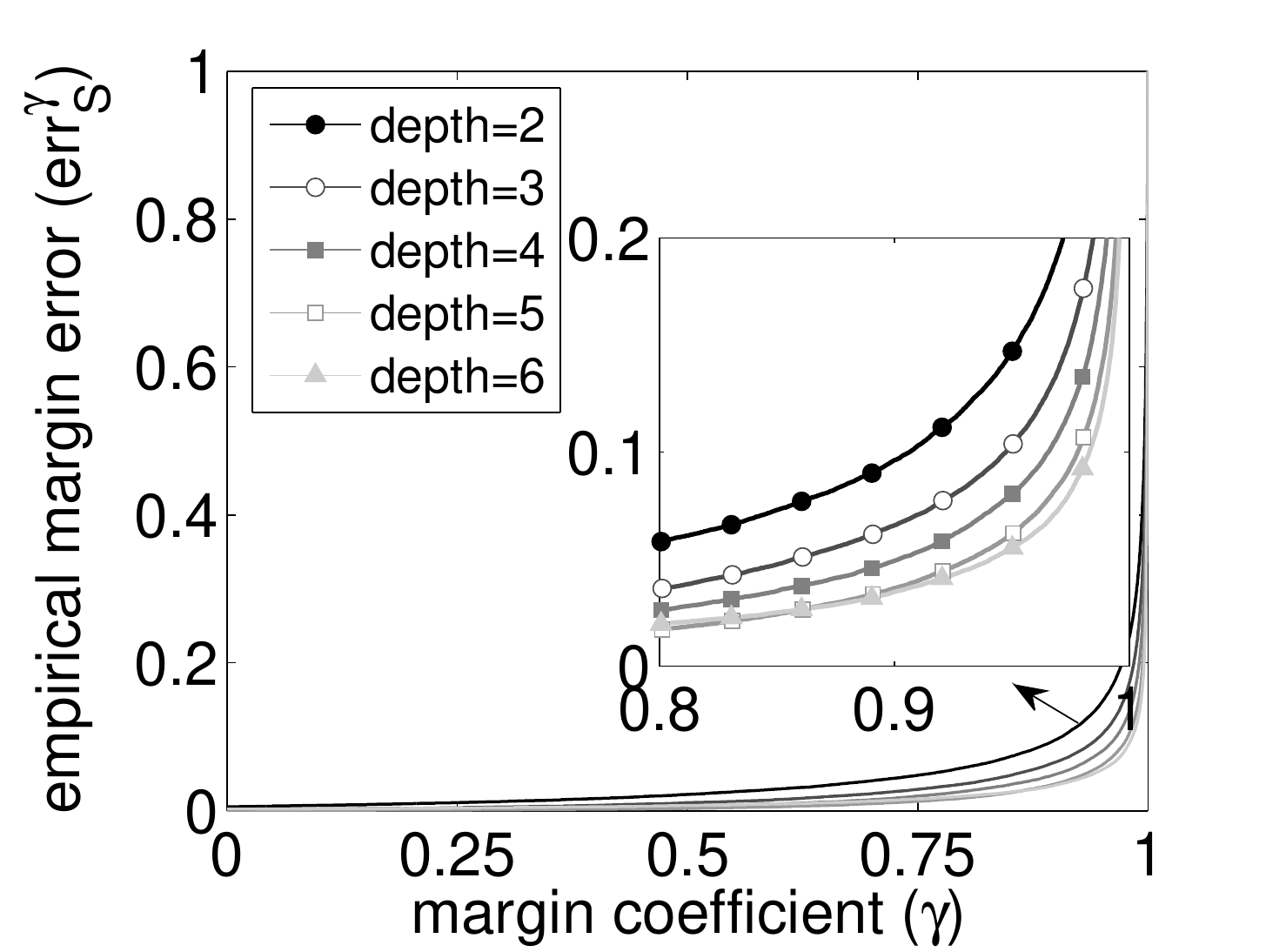}
			\label{subfig:mnist_margin_depth}
		}
		\subfloat[CIFAR-10]
		{
			\includegraphics[width=0.47\columnwidth]{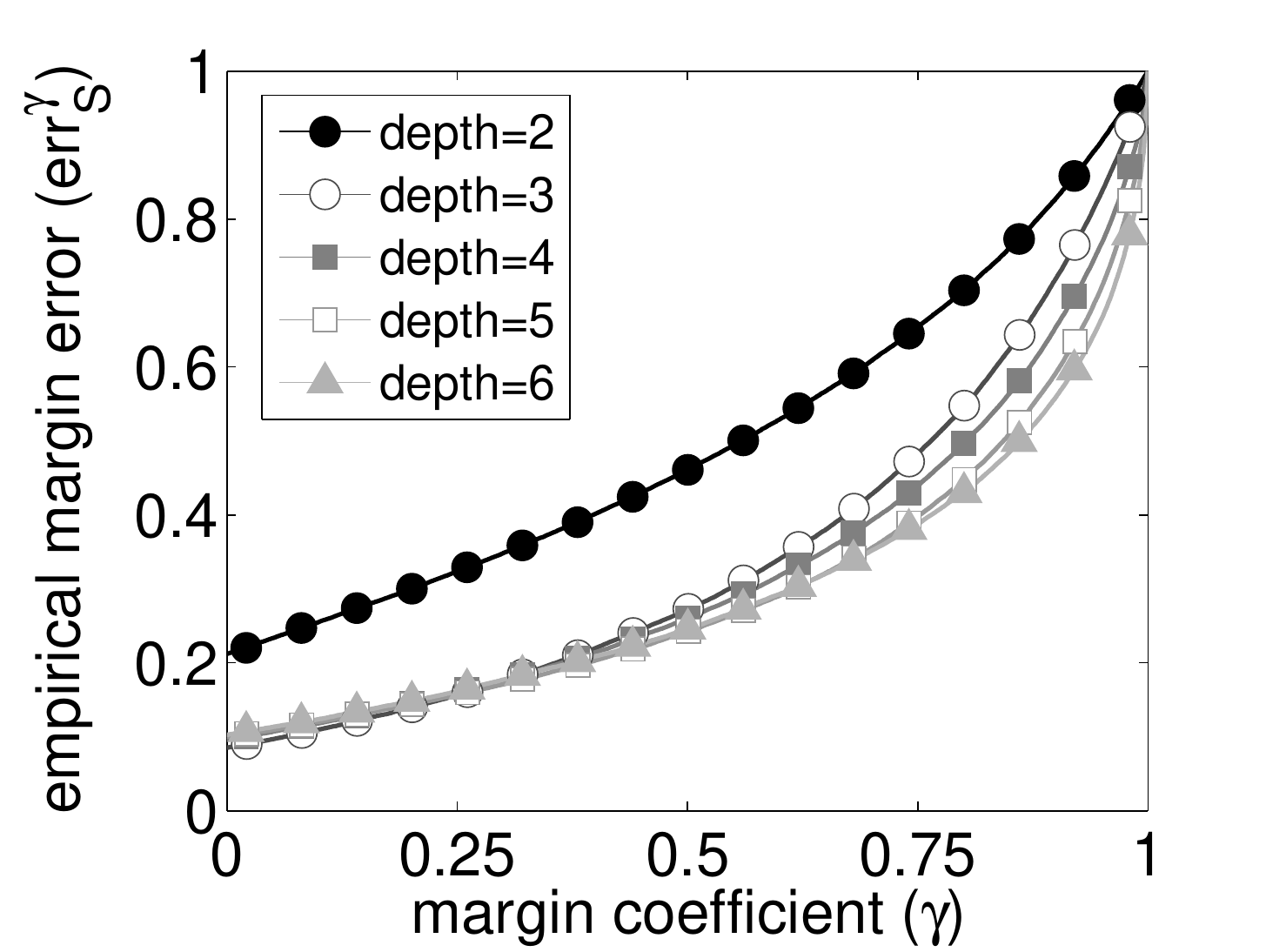}
			\label{subfig:cifar_margin_depth}
		} 
		\caption{The influence of depth on empirical margin error.}
		\label{fig:depth_margin}
	\end{figure}	
	\subsection{Discussions} 
	Based on discussions in previous two subsections, we can see that when the depth $L$ of DNN increases, (1) the RA term in margin bound will increase (according to Theorem~\ref{thm:generalization_nn}); (2) the empirical margin error in margin bound will decrease since deeper nets have larger representation power (according to Theorem~\ref{thm:multi_class_complexity}). As a consequence, we can come to the conclusion that, for DNN with restricted number of hidden units, arbitrarily increasing depth is not always good since there is a clear tradeoff between its positive and negative impacts on test error. In other words, with the increasing depth, the test error of DNN may first decrease, and then increase. 
	
	Actually this theoretical pattern is consistent with our empirical observations on different datasets. We used the same experimental setting as that in the subsection~\ref{subsec:margin_exp} and repeated the training of DNN (with different random initializations) for 5 times. Figure~\ref{fig:depth_test_error} reports the average and minimum test error of $5$ learned models. We can observe that as the depth increases, the test error first decreases (probably because increased representation power overwhelms increased RA capacity); and then increase (probably because RA capacity increases so quickly that representation power cannot compensate for negative impact of increased capacity). 
	\begin{figure}
		\centering
		\subfloat[MNIST]
		{
			\includegraphics[width=0.48\columnwidth]{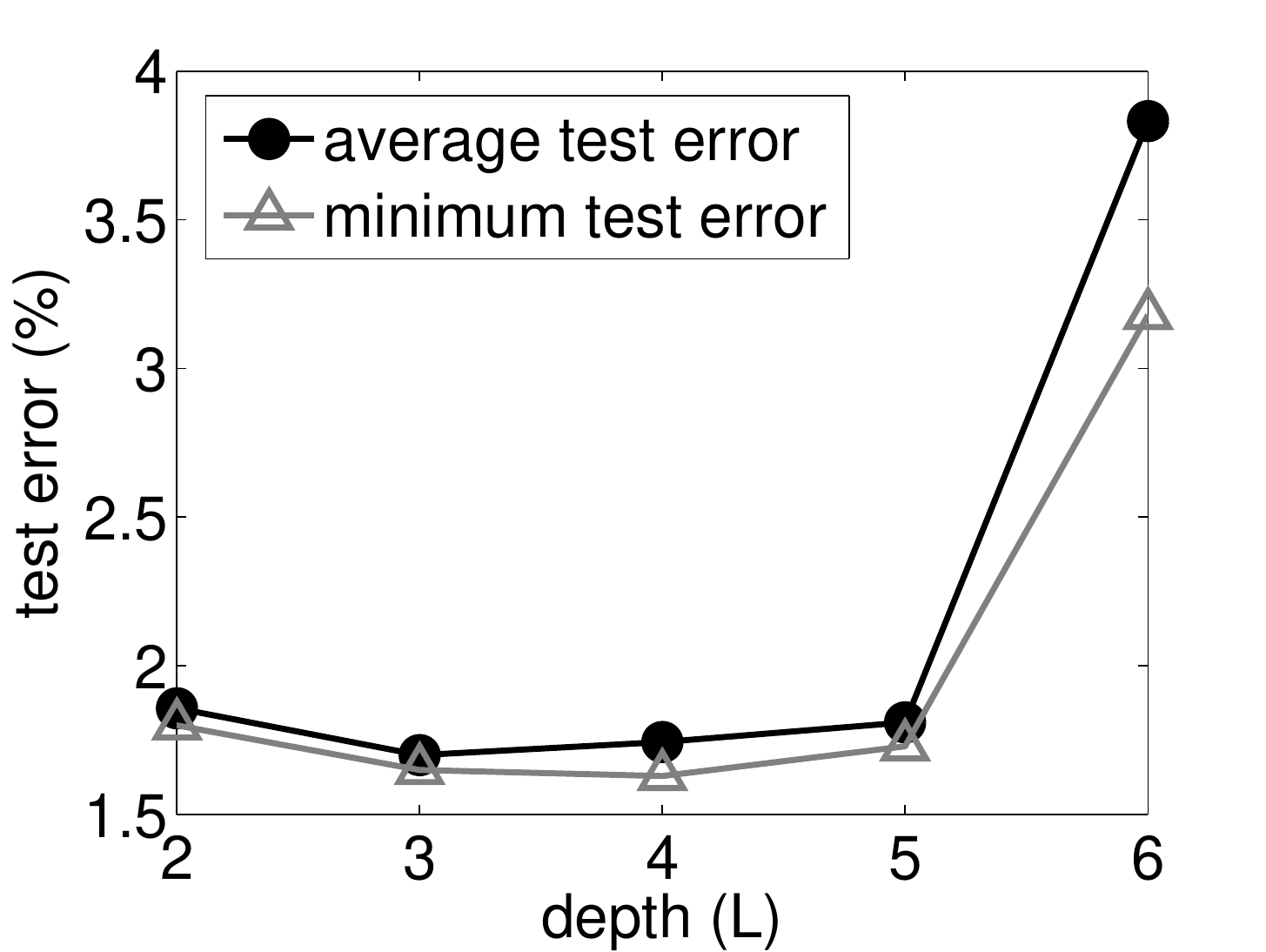}
			\label{subfig:mnist_test_depth}
		}
		\subfloat[CIFAR-10]
		{
			\includegraphics[width=0.48\columnwidth]{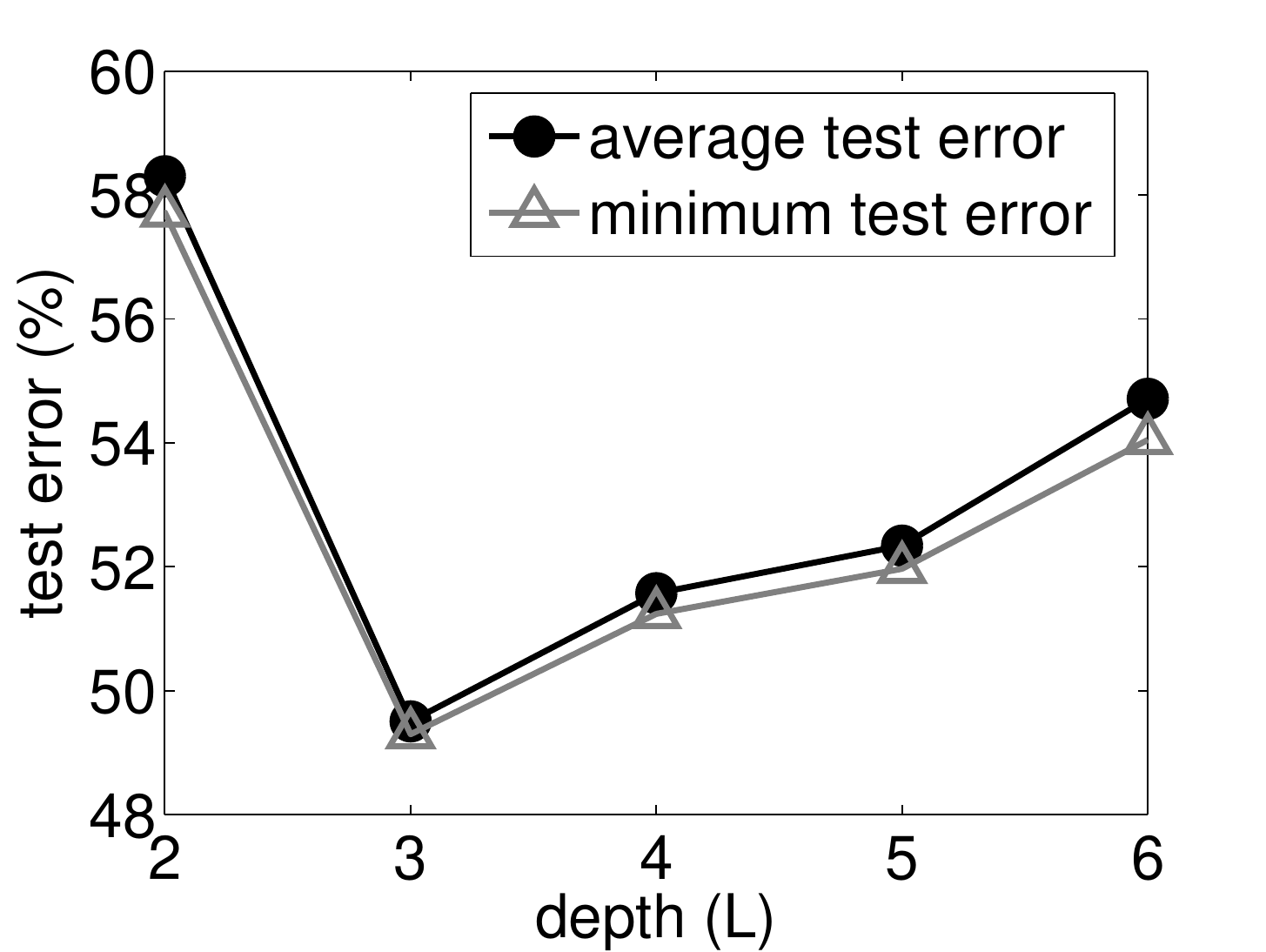}
			\label{subfig:cifar_test_depth}
		}
		\caption{The influence of depth on test error.}
		\label{fig:depth_test_error} 
	\end{figure}
	
	\section{Large Margin Deep Neural Networks}\label{sec:alg} 
	From the discussions in Section~\ref{sec:generalization}, we can see that one may have to pay the cost of larger RA capacity when trying to obtain better representation power by increasing the depth of DNN (not to mention that the effective training of very deep neural networks is highly non-trivial~\cite{glorot2010understanding,srivastava2015training}). Then a nature question is whether we can avoid this tradeoff, and achieve good test performance in an alternative way. 
	
	To this end, let us revisit the positive impact of depth: it actually lies in that deeper neural networks tend to have larger representation power and thus smaller empirical margin error. Then the question is: can we directly minimize empirical margin error? Our answer to this question is yes, and our proposal is to add a margin-based penalty term to current loss function. In this way, we should be able to effectively tighten margin bound without manipulating the depth. 
	
	One may argue that widely used loss functions (e.g., cross entropy loss and hinge loss) in DNN are convex surrogates of margin error by themselves, and it might be unnecessary to introduce an additional margin-based penalty term. However, we would like to point out that unlike hinge loss for SVM or exponential loss for Adaboost, which have theoretical guarantee for convergence to margin maximizing separators as the regularization vanishes~\cite{rosset2003margin}, there is no optimization consistency guarantee for these losses used in DNN since neural networks are highly non-convex. Therefore, it makes sense to explicitly add a margin-based penalty term to loss function, in order to further reduce empirical margin error during training process.
	
	\subsection{Algorithm Description}\label{subsec:alg_description} 	
	We propose adding two kinds of margin-based penalty terms to the original cross entropy loss\footnote{Although we take the most widely-used cross entropy loss as example, these margin-based penalty terms can also be added to other loss functions.}. The first penalty term is the gap between the upper bound of margin (i.e., 1)\footnote{Please note that, after softmax operation, the outputs are normalized to $[0,1]$} and the margin of the sample (i.e., $\rho(f; x, y)$). The second one is the average gap between upper bound of margin and the difference between the predicted output for the true category and those for all the wrong categories. It can be easily verified that the second penalty term is an upper bound of the first penalty term. Mathematically, the penalized loss functions can be described as follows (for ease of reference, we call them $C_{1}$ and $C_{2}$ respectively): for model $f$, sample $x,y$,
	\begin{small}
		\begin{align*}
		C_{1}(f;x,y)= & C(f;x,y) + \lambda \Big(1-\rho(f;x,y)\Big)^2,\\
		C_{2}(f;x,y)= & C(f;x,y) \\
		&+\frac{\lambda}{K-1}\sum_{k\neq y}\Big(1-(f(x,y)-f(x,k))\Big)^2.
		\end{align*}		
	\end{small} 
	We call the algorithms that minimize the above new loss functions \emph{large margin DNN algorithms} (LMDNN). For ease of reference, we denote LMDNN minimizing $C_1$ and $C_2$ as LMDNN-$C_1$ and LMDNN-$C_2$ respectively, and the standard DNN algorithms minimizing $C$ as DNN-$C$. To train LMDNN, we also employ the back propagation method.
	\subsection{Experimental Results}\label{subsec:alg_exp} 
	\begin{table}
		\begin{center}
			\begin{small}
				\begin{tabular}{cccc}
					\hline
					\hline 
					& MNIST & CIFAR-10 \\
					\hline 
					DNN-$C$ (\%) & $0.899\pm0.038$ & $18.339\pm0.336$ \\
					LMDNN-$C_1$ (\%) & $\mathbf{0.734\pm0.046}$ & $\mathbf{17.598\pm0.274}$ \\ 
					LMDNN-$C_2$ (\%) &$0.736\pm0.041$&$17.728\pm0.283$ \\
					\hline
					\hline
				\end{tabular}
			\end{small}
		\end{center}
		\caption{Test error (\%) of DNN-$C$ and LMDNNs.}
		\label{tab:loss_test_error}
	\end{table} 
	\begin{figure}[t]
		\centering
		\subfloat[MNIST]
		{
			\includegraphics[width=0.48\columnwidth]{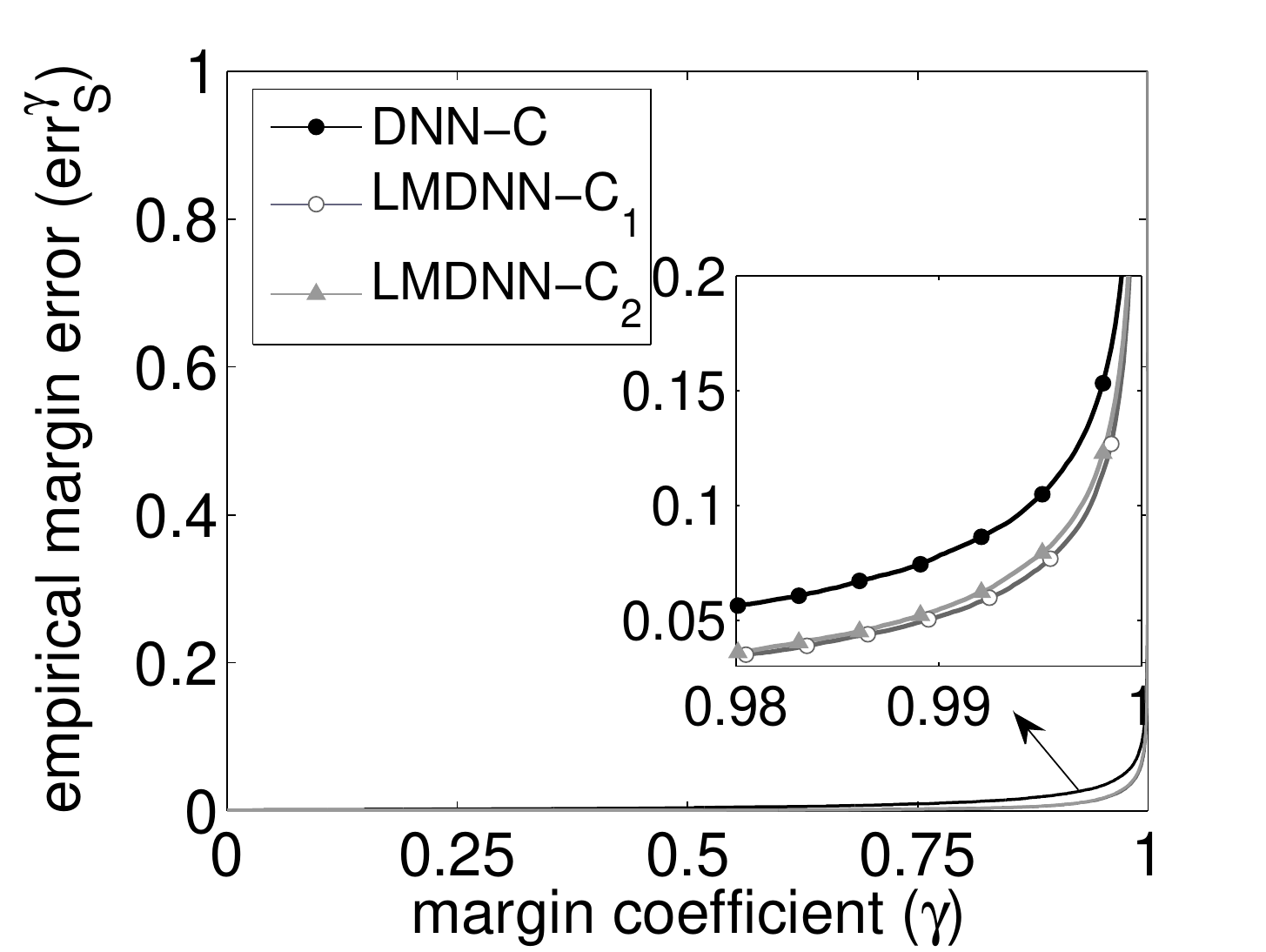}
			\label{subfig:mnist_loss_margin} 
		}
		\subfloat[CIFAR-10]
		{
			\includegraphics[width=0.48\columnwidth]{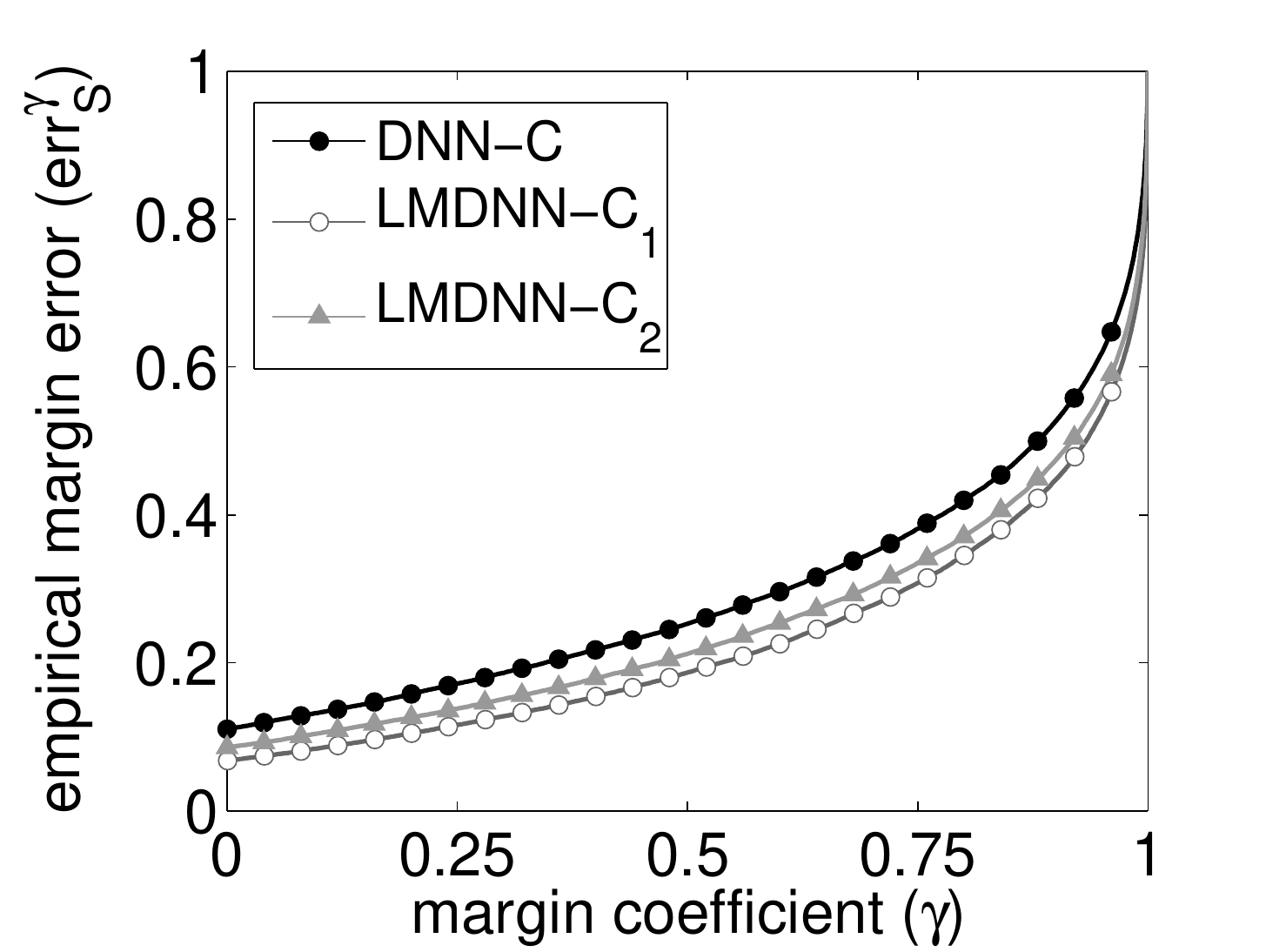}
			\label{subfig:cifar_loss_margin}
		}
		\caption{Empirical margin error of LMDNNs.}
		\label{fig:comp_margin}
	\end{figure}
	\begin{figure}[t]
		\centering
		\subfloat[MNIST]
		{
			\includegraphics[width=0.48\columnwidth]{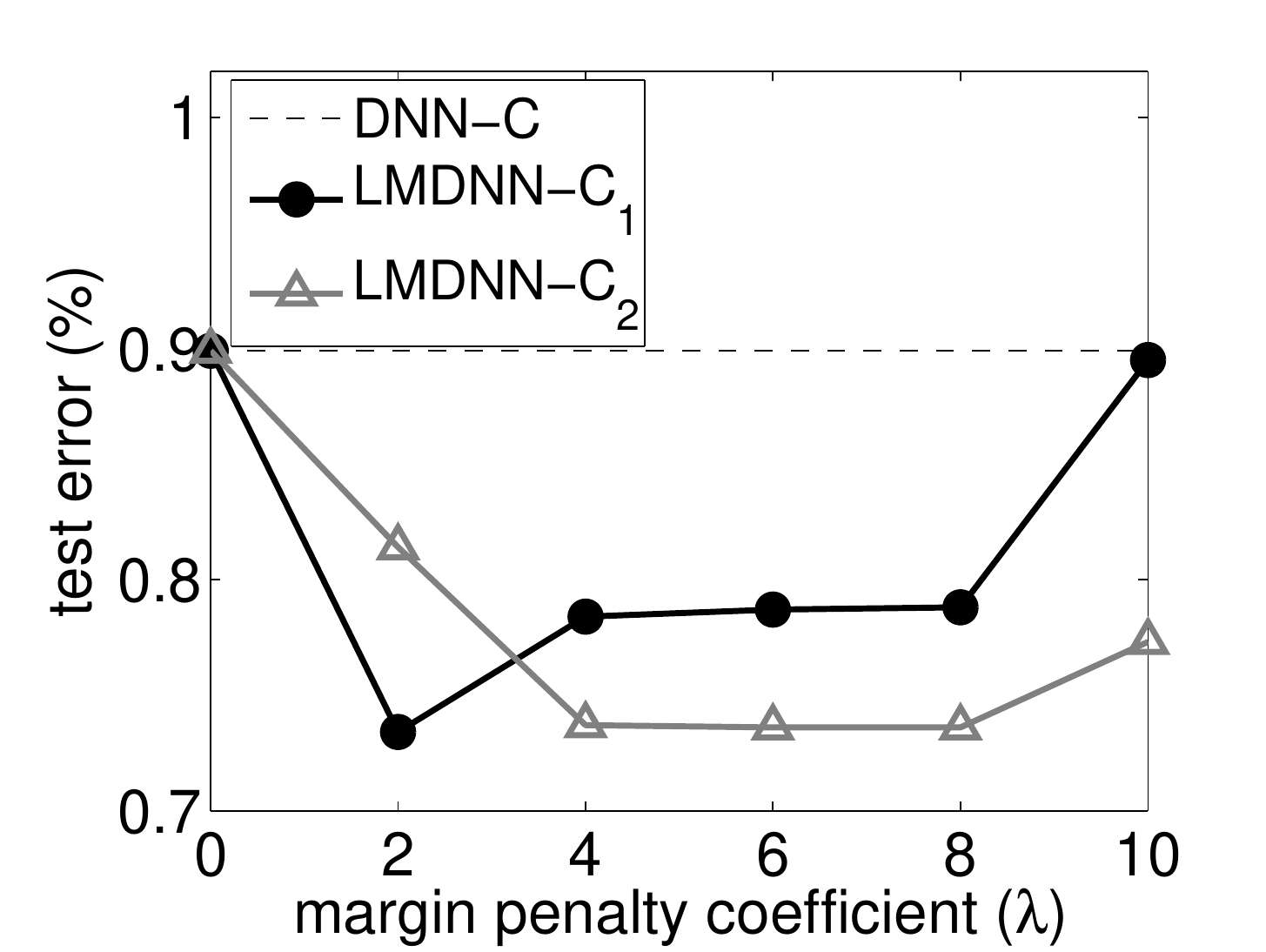}
			\label{subfig:mnist_lambda_alg} 
		}
		\subfloat[CIFAR-10]
		{
			\includegraphics[width=0.48\columnwidth]{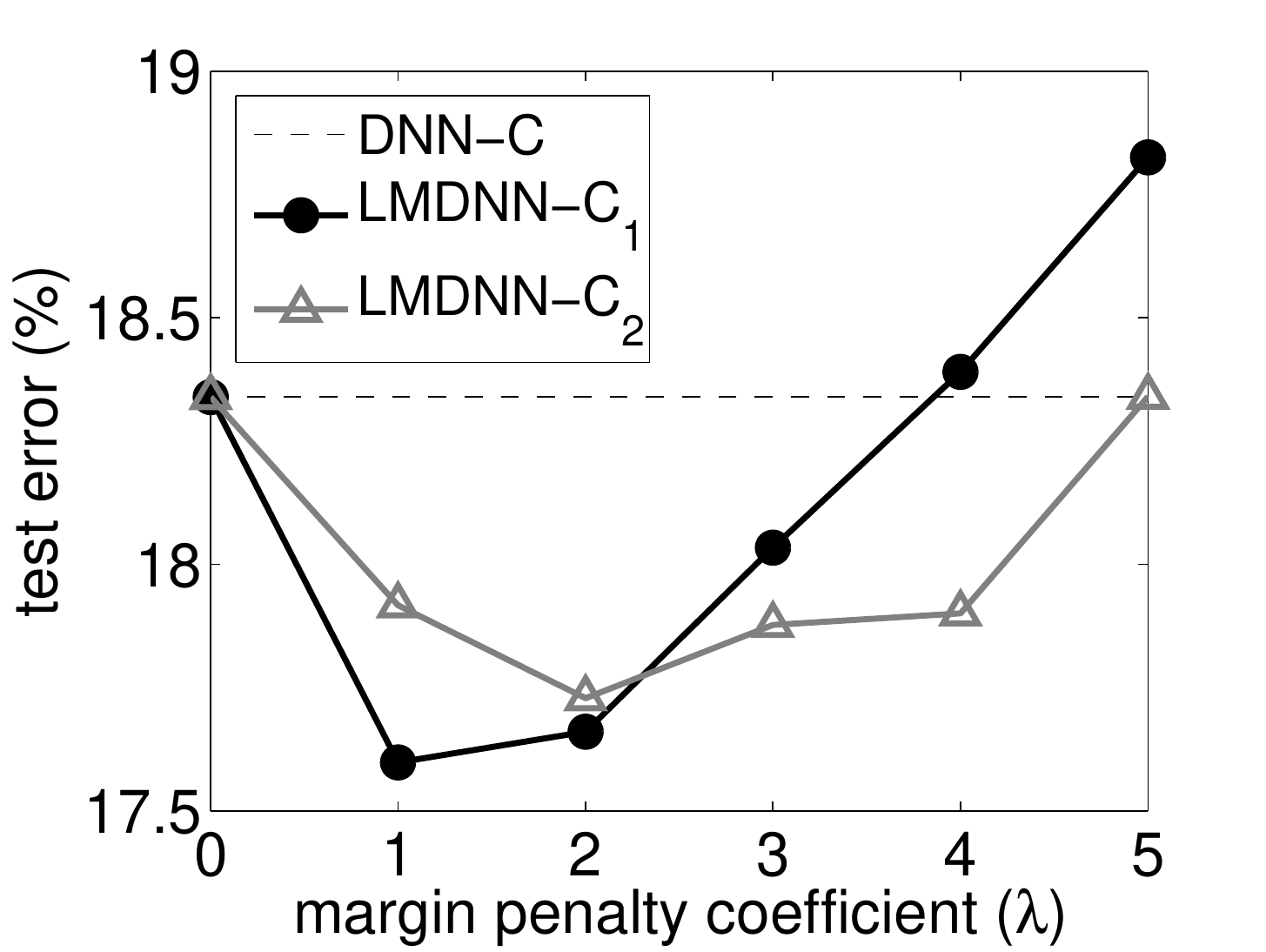} 
			\label{subfig:cifar_lambda_alg}
		} 
		\caption{Test error of LMDNNs with different $\lambda$.}
		\label{fig:lambda_alg}
	\end{figure}
	Now we compare the performances of LMDNNs with DNN-$C$. We used well-tuned network structures in the Caffe~\cite{jia2014caffe} tutorial (i.e., LeNet\footnote{\url{http://caffe.berkeleyvision.org/gathered/examples/mnist.html}} for MNIST and AlexNet\footnote{\url{http://caffe.berkeleyvision.org/gathered/examples/cifar10.html}} for CIFAR-10) and all the tuned hyper parameters on the validation set.
	
	Each model was trained for 10 times with different initializations. Table~\ref{tab:loss_test_error} shows mean and standard deviation of test error over the 10 learned models for DNN-$C$ and LMDNNs after tuning margin penalty coefficient $\lambda$. We can observe that, on both MNIST and CIFAR-10, LMDNNs achieve significant performance gains over DNN-$C$. In particular, LMDNN-$C_1$ reduce test error from $0.899\%$ to $0.734\%$ on MNIST and from $18.399\%$ to $17.598\%$ on CIFAR-10; LMDNN-$C_2$ reduce test error from $0.899\%$ to $0.736\%$ on MNIST and from $18.399\%$ to $17.728\%$ on CIFAR-10.
	
	To further understand the effect of adding margin-based penalty terms, we plot empirical margin errors of DNN-$C$ and LMDNNs in Figure~\ref{fig:comp_margin}. We can see that by introducing margin-based penalty terms, LMDNNs indeed achieve smaller empirical margin errors than DNN-$C$. Furthermore, the models with smaller empirical margin errors really have better test performances. For example, LMDNN-$C_{1}$ achieved both smaller empirical margin error and better test performance than LMDNN-$C_{2}$. This is consistent with Theorem~\ref{thm:multi_class_complexity}, and in return indicates reasonability of our theorem.
	
	We also report mean test error of LMDNNs with different margin penalty coefficient $\lambda$ (see Figure \ref{fig:lambda_alg}). In the figure, we use dashed line to represent mean test error of DNN-$C$ (corresponding to $\lambda=0$). From the figure, we can see that on both MNIST and CIFAR-10, (1) there is a range of $\lambda$ where LMDNNs outperform DNN-$C$; (2) although the best test performance of LMDNN-$C_{2}$ is not as good as that of LMDNN-$C_{1}$, the former has a broader range of $\lambda$ that can outperform DNN-$C$ in terms of the test error. This indicates the value of using LMDNN-$C_{2}$: it eases the tuning of hyper parameter $\lambda$; (3) with increasing $\lambda$, test error of LMDNNs will first decrease, and then increase. When $\lambda$ is in a reasonable range, LMDNNs can leverage both good the optimization property of cross entropy loss in training process and the effectiveness of margin-based penalty term, and thus achieve good test performance. When $\lambda$ becomes too large, margin-based penalty term dominates cross entropy loss. Considering that margin-based penalty term may not have good optimization property as cross entropy loss in the training process, the drop of test error is understandable. 
	
	\section{Conclusion and Future Work}\label{sec:conclusion}
	In this work, we have investigated the role of depth in DNN from the perspective of margin bound. We find that while the RA term in margin bound is increasing w.r.t. depth, the empirical margin error is decreasing instead. Therefore, arbitrarily increasing the depth might not be always good, since there is a tradeoff between the positive and negative impacts of depth on test performance of DNN. Inspired by our theory, we propose two large margin DNN algorithms, which achieve significant performance gains over standard DNN algorithm. In the future, we plan to study how other factors influence the test performance of DNN, such as unit allocations across layers and regularization tricks. We will also work on the design of effective algorithms that can further boost the performance of DNN.
	\section{Acknowledgments}
	Liwei Wang was partially supported by National Basic Research Program of China (973 Program) (grant no. 2015CB352502), NSFC(61573026), and a grant from MOE-Microsoft Laboratory of Statistics of Peking University. Xiaoguang Liu was partially supported by NSF of China (61373018, 11301288, 11450110409) and Program for New Century Excellent Talents in University (NCET130301).
	\bibliography{deepnets}
	\bibliographystyle{aaai}
	\begin{appendices}
		\section{Experiment Settings in Section 3.2}
		The MNIST dataset (for handwritten digit classification) consists of $28\times28$ black and white images, each containing a digit $0$ to $9$. There are $60$k training examples and $10$k test examples in this dataset. The CIFAR-10 dataset (for object recognition) consists of $32\times32$ RGB images, each containing an object, e.g., cat, dog, or ship. There are $50$k training examples and $10$k test examples in this dataset. For each dataset, we divide the $10$k test examples into two subsets of equal size, one for validation and the other for testing. In each experiment, we use standard sigmoid activation in hidden layers and train neural networks by mini-batch SGD with momentum and weight decay. All the hyper-parameters are tuned on the validation set.
		
		To investigate the influence of the network depth $L$, we train fully-connected DNN models with different depths and restricted number of hidden units. For simplicity and also following many previous works~\cite{simard2003best,hinton2012improving,glorot2011deep,ba2014deep}, we assume that each hidden layer has the same number of nodes in the experiment. Specifically, for MNIST and CIFAR-10, the DNN models with depth $2$, $3$, $4$, $5$ and $6$ respectively have $3000$, $1500$, $1000$, $750$ and $600$ units in each hidden layer when the total number of hidden units is $3000$.
		\section{Experimental Settings in Section 4.2}
		For data pre-processing, we scale the pixel values in MNIST to $[0,1]$, and subtract the per-pixel mean computed over the training set from each image in CIFAR-10. On both datasets, we do not use data augmentation for simplicity.
		
		For network structure, we used the well-tuned neural network structures as given in the Caffe tutorial (i.e., LeNet for MNIST and AlexNet) for CIFAR-10).
		
		For the training process, the weights are initialized randomly and updated by mini-batch SGD. We use the model in the last iteration as our final model. For DNN-$C$, all the hyper parameters are set by following Caffe tutorial. For LMDNNs, all the hyper parameters are tuned to optimal on the validation set. Finally, we find that by using the following hyper parameters, both DNN-$C$ and LMDNNs can achieve best performance as we reported. For MNIST, we set the batch size as $64$, the momentum as $0.9$, and the weight decay coefficient as $0.0005$. Each neural network is trained for $10$k iterations and the learning rate in each iteration $T$ decreases by multiplying the initial learning rate with a factor of $(1 + 0.0001T) ^ {-0.75}$. For CIFAR-10, we set the batch size as $100$, the momentum as $0.9$, and the weight decay coefficient as $0.004$. Each neural network is trained for $70$k iterations. The learning rate is set to be $10^{-3}$ for the first $60$k iterations, $10^{-4}$ for the next $5$k iterations, and $10^{-5}$ for the other $5$k iterations.
	\end{appendices}	
\end{document}